\documentclass{article}
\usepackage{arxiv}
\usepackage{setspace}
\usepackage{microtype}
\usepackage{graphicx}
\graphicspath{{figures/}}
\newlength{\figwidth}
\usepackage{booktabs}
\usepackage[numbers,sort&compress]{natbib}
\definecolor{linkblue}{HTML}{4A6FA5}    % muted steel blue
\usepackage[colorlinks=true,citecolor=linkblue,linkcolor=linkblue,urlcolor=linkblue]{hyperref}

\usepackage{amsmath, amssymb, amsfonts, amsthm}
\usepackage{mathtools}

\usepackage{subcaption}
\usepackage{placeins}
\usepackage{wrapfig}

\usepackage{url}
\usepackage{enumitem}
\usepackage{needspace}
\usepackage{etoolbox}
\AtBeginEnvironment{thebibliography}{\sloppy\raggedright\setlength{\emergencystretch}{2em}}

\usepackage[most]{tcolorbox}

\definecolor{thmback}{HTML}{F0F4F8}    % very light blue-gray

\theoremstyle{plain}
\newtheorem{theorem}{Theorem}

\theoremstyle{definition}

\theoremstyle{remark}

\tcolorboxenvironment{theorem}{
  colback=thmback, colframe=thmback,
  boxrule=0pt, arc=2pt, left=8pt, right=8pt, top=6pt, bottom=6pt,
  before skip=8pt, after skip=8pt, breakable
}
\tcolorboxenvironment{definition}{
  colback=thmback, colframe=thmback,
  boxrule=0pt, arc=2pt, left=8pt, right=8pt, top=6pt, bottom=6pt,
  before skip=8pt, after skip=8pt, breakable
}

\newcommand{\simplex}{\Delta}

\title{{\bfseries When Is Collective Intelligence a Lottery?\\ Multi-Agent Scaling Laws for Memetic Drift in LLMs}}

\author{
  \textbf{Hidenori Tanaka}\\
  CBS--NTT Program in Physics of Intelligence, Center for Brain Science, Harvard University\\
  Physics of Artificial Intelligence Group, NTT Research, Inc.\\
  tanaka@g.harvard.edu
}

\date{}

\begin{document}
\raggedbottom

\maketitle

\begin{abstract}
Multi-agent systems powered by large language models (LLMs) are increasingly deployed in settings that shape consequential decisions, both directly and indirectly. Yet it remains unclear whether their outcomes reflect collective reasoning, systematic bias, or mere chance. Recent work has sharpened this question with naming games, showing that even when no individual agent favors any label \emph{a priori}, populations rapidly break symmetry and reach consensus. Here, we reveal the mechanism by introducing a minimal model, Quantized Simplex Gossip (QSG), and trace the microscopic origin of this agreement to \emph{mutual in-context learning}. In QSG, agents maintain internal belief states but learn from one another's sampled outputs, so one agent's arbitrary choice becomes the next agent's evidence and can compound toward agreement. By analogy with neutral evolution, we call this sampling-driven regime \emph{memetic drift}. QSG predicts a crossover from a drift-dominated regime, where consensus is effectively a lottery, to a selection regime, where weak biases are amplified and shape the outcome. We derive scaling laws for drift-induced polarization as a function of population size, communication bandwidth, in-context adaptation rate, and agents' internal uncertainty, and we validate them in both QSG simulations and naming-game experiments with LLM populations. Together, these results provide a framework for studying the collective mechanisms of social representation formation in multi-agent systems.
\end{abstract}

\section{Introduction}
Collective intelligence arises from interactions among individual agents, but its outcomes often cannot be understood from any one agent alone and instead require principles at the level of the collective \cite{Castellano2009RMP}.
Human language and shared conventions are canonical examples, emerging from repeated exchanges in which ideas are transmitted, modified, and sometimes lost \cite{Steels1995SpatialVocab,Baronchelli2006Sharp,Baronchelli2008NG}.
Recent advances in large language models (LLMs) have made similar collective dynamics relevant beyond human populations. LLM systems are increasingly studied as interacting populations \cite{Chuang2023LLMOpinion,Brockers2025Disentangling,Zhao2024Electoral,Kaesberg2025VotingConsensus,Becker2025MALLM}, and they are also being developed and evaluated for consequential settings including law, finance, healthcare, policy, and scientific discovery \cite{DBLP:journals/corr/abs-2412-11063,shi2024tradingagentsmultiagentsllmfinancial,Sreedhar2025ProsocialPolicy,kim2024mdagentsadaptivecollaborationllms,Gottweis2025AICoScientist}.
As these systems move toward consequential applications, a central question emerges: when an LLM population reaches consensus, does that outcome reflect collective reasoning, systematic bias, or stochastic sampling?

A growing empirical literature studies multi-agent LLM systems through debate, collaboration, and opinion dynamics \cite{Du2024MultiagentDebate,Zhang2024CollaborationMechanisms,Chuang2023LLMOpinion,Brockers2025Disentangling}.
These studies show that LLM populations can exhibit nontrivial collective behavior, but most work focuses on practically relevant interaction settings, so systematic accounts of how microscopic interactions shape collective dynamics remain limited.
Answering the question above therefore requires a controlled setting in which coordination can be isolated from practical details.

One such setting is the naming game, a canonical synthetic task for convention formation \cite{Steels1995SpatialVocab,Baronchelli2006Sharp}. In a naming game, agents repeatedly propose labels for a shared referent and adapt through interaction, so population-level conventions can emerge from local exchanges. Recent work has shown that LLM populations in naming games can spontaneously break symmetry and reach consensus even when no individual agent favors any label \emph{a priori} \cite{Ashery2025SciAdv}, and that population size can strongly shape how deterministic and biased the resulting coordination becomes \cite{Flint2025GroupSize}. Here we isolate a neutral naming-game setting without explicit rewards or ground truth, so any coordination must emerge from interaction-driven in-context learning alone. These results sharpen the problem: if the population begins neutral, what breaks symmetry, and what competing forces drive the population toward consensus? Figure~\ref{fig:symmetry_breaking_main} shows this symmetry breaking in a representative LLM naming game.

\begin{figure}[t]
\centering
\begin{subfigure}[t]{0.32\linewidth}
    \centering
    \includegraphics[width=\linewidth]{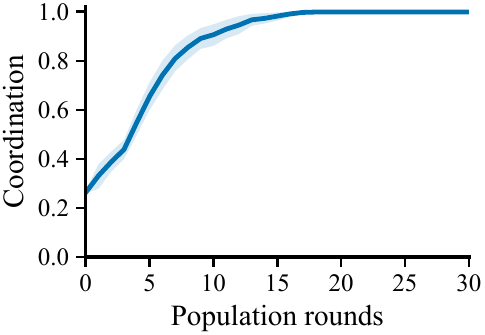}
    \caption{Mean coordination.}
    \label{fig:symmetry_breaking_main_coord}
\end{subfigure}\hfill
\begin{subfigure}[t]{0.31\linewidth}
    \centering
    \includegraphics[width=\linewidth]{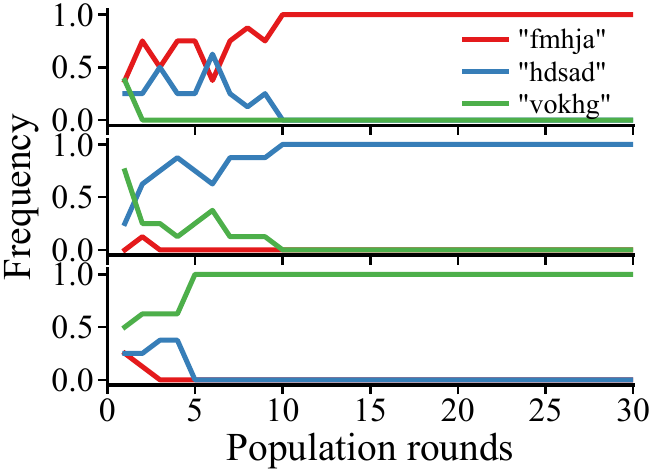}
    \caption{Representative trajectories.}
    \label{fig:symmetry_breaking_main_traj}
\end{subfigure}\hfill
\begin{subfigure}[t]{0.33\linewidth}
    \centering
    \includegraphics[width=\linewidth]{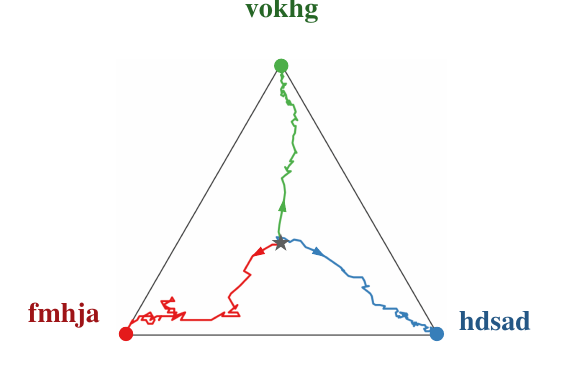}
    \caption{Winner-conditioned simplex paths.}
    \label{fig:symmetry_breaking_main_simplex}
\end{subfigure}
\caption{\textbf{Symmetry breaking in an LLM naming game} (GPT-4o, $N=24$, $K=3$).
(a) Mean coordination across trials (95\% CI).
(b) Representative label-frequency trajectories from three trials, one for each eventual winner.
(c) Mean simplex trajectories conditioned on the eventual winning label.}
\label{fig:symmetry_breaking_main}
\end{figure}

\begin{wrapfigure}[21]{r}{0.39\linewidth}
\vspace{-1\baselineskip}
\centering
\includegraphics[width=0.98\linewidth]{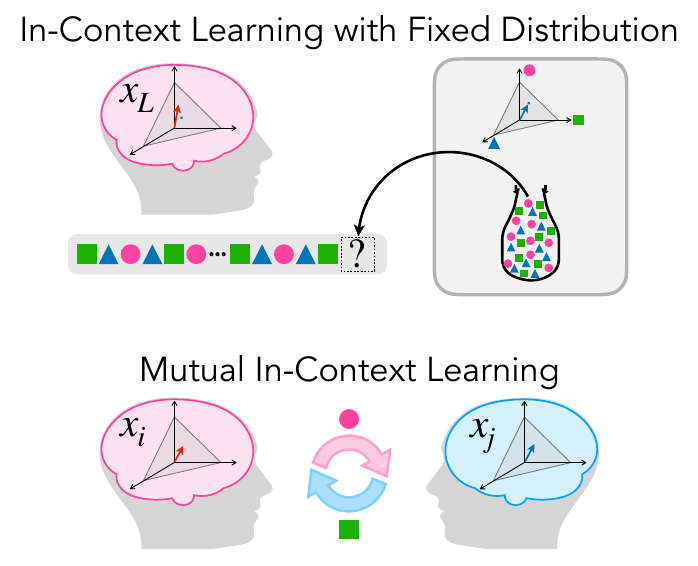}
\caption{\textbf{Individual vs.\ mutual in-context learning.}
Top: standard in-context learning, where a single agent updates from i.i.d.\ tokens drawn from a stationary external distribution.
Bottom: mutual in-context learning, where agents update from one another's sampled outputs, so the population becomes its own evolving data source.}
\label{fig:individual_vs_mutual}
\end{wrapfigure}

The key to understanding this symmetry breaking is that, unlike standard in-context learning, the learning signal in multi-agent coordination is generated by the population itself (Fig.~\ref{fig:individual_vs_mutual}).
In the usual single-agent setting, an agent updates from tokens drawn from a fixed external distribution \cite{Akyurek2022ICL,vonOswald2023ICLGD,Dai2023WhyGPT,Park2025ICLRRepresentations,Park2024CompetitionDynamics}. Here, by contrast, each agent updates from another agent's \emph{sampled output}, which is itself a stochastic draw from a changing internal state.
We refer to this feedback loop as \textbf{mutual in-context learning}: agents learn from one another's sampled outputs as their internal states co-evolve.
This evidence-accumulation perspective is also consistent with recent belief-dynamics accounts of in-context learning \cite{Bigelow2025BeliefDynamics}.
Under mutual in-context learning, an arbitrary early sample can be reused as evidence and amplified into population-wide agreement; in the neutral setting we study, that sampling noise alone is enough to break symmetry and drive consensus.

This perspective naturally connects the problem to evolutionary dynamics, where collective outcomes arise from the balance between selection and stochastic drift \cite{Kimura1968Rate,Kimura1983NeutralTheory,Moran1958RandomProcesses,CliffordSudbury1973SpatialConflict,HolleyLiggett1975VoterModel,Liggett1999Interacting}.
In that language, neutral drift denotes fixation driven by stochastic sampling, while selection denotes systematic asymmetry that favors one outcome over another.
Applying this language to memes, understood here as culturally transmitted units, we call the neutral coordination regime \textbf{memetic drift}. As in population genetics, memetic drift provides a null baseline against which selection can be detected \cite{Kimura1983NeutralTheory,Chotibut2017FluctuatingPop,Kingman1982Coalescent}; without such a baseline, apparent \emph{collective intelligence} may be hard to distinguish from amplified sampling noise.

To study this regime quantitatively, we introduce Quantized Simplex Gossip (QSG), a minimal, analytically tractable model of continuous beliefs, quantized communication, and in-context adaptation.
At a high level, QSG captures a crossover between a drift-dominated regime, where the winner is largely decided by luck, and a selection-dominated regime, where weak biases are collectively amplified.
The location of that crossover depends on population size, communication bandwidth, and adaptation strength. Larger populations and higher-bandwidth communication suppress drift. Stronger adaptation speeds the dynamics overall, but for a fixed weak asymmetry it strengthens drift relative to bias and makes the same bias less decisive.
Recent population-size effects in LLM naming games highlight the need to understand different coordination regimes \cite{Flint2025GroupSize}. Figure~\ref{fig:luckiest_vs_fittest} previews the drift--selection crossover.
Framed this way, the question of whether consensus is a lottery becomes a quantitative scaling question.

\begin{figure}[t]
\centering
\begin{subfigure}[t]{0.245\linewidth}
    \centering
    \includegraphics[width=\linewidth]{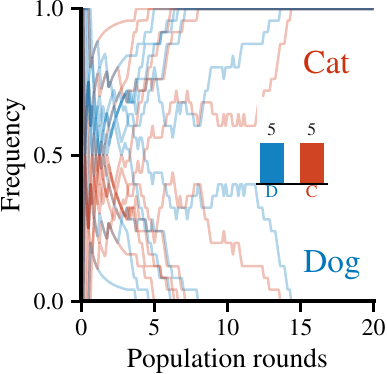}
    \caption{Drift at $N=8$.}
\label{fig:luckiest_vs_fittest_small_n}
\end{subfigure}\hfill
\begin{subfigure}[t]{0.245\linewidth}
    \centering
    \includegraphics[width=\linewidth]{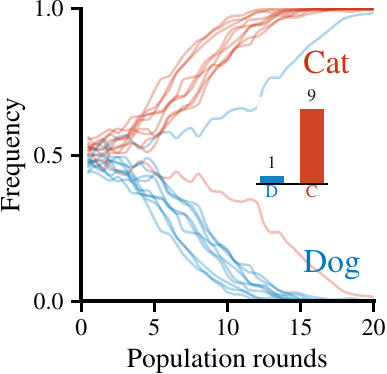}
    \caption{Selection at $N=800$.}
\label{fig:luckiest_vs_fittest_large_n}
\end{subfigure}\hfill
\begin{subfigure}[t]{0.245\linewidth}
    \centering
    \includegraphics[width=\linewidth]{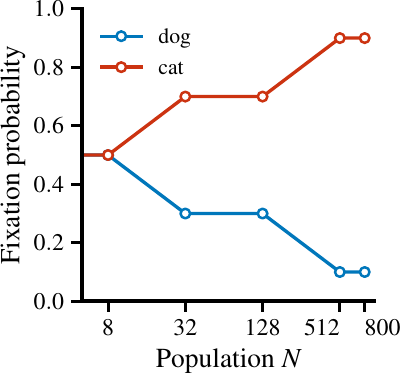}
    \caption{Fixation vs.\ $N$.}
\label{fig:luckiest_vs_fittest_fixation}
\end{subfigure}\hfill
\begin{subfigure}[t]{0.245\linewidth}
    \centering
    \includegraphics[width=\linewidth]{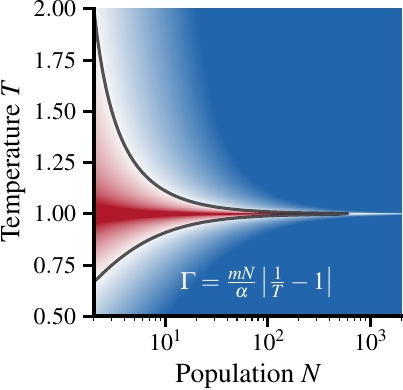}
    \caption{Temperature crossover.}
\label{fig:luckiest_vs_fittest_phase}
\end{subfigure}
\caption{\textbf{Drift--selection crossover in multi-agent naming from GPT-4o experiments and QSG theory.}
Panels (a)--(c) show two-label (Dog, Cat) naming-game experiments with GPT-4o using random ordered speaker--listener pairs for a fixed referent $r$.
(a) At $N=8$, runs show substantial run-to-run variability, and the inset shows winner counts across trials.
(b) At $N=800$, a weak asymmetry consistently selects the same winner (Cat), again reflected in the inset counts.
(c) Fixation probability versus $N$ shows a finite-size crossover.
(d) The corresponding tempered-sampling crossover diagram derived from QSG theory in $(T,mN/\alpha)$; the black curve marks $\Gamma_T=1$, where $\Gamma_T=\frac{mN}{\alpha}\left|\frac{1}{T}-1\right|$.}
\label{fig:luckiest_vs_fittest}
\end{figure}

\Needspace{12\baselineskip}
Specifically, we make the following contributions:
\begin{enumerate}
    \item \textbf{Mutual in-context learning and memetic drift.} We identify the sampling-driven mechanism underlying symmetry breaking and consensus under neutrality in LLM populations, trace it to \emph{mutual in-context learning}, and formalize the resulting regime as \emph{memetic drift}.
    \item \textbf{Quantized Simplex Gossip (QSG).} We introduce QSG, a minimal and tractable model of continuous beliefs, quantized communication, and in-context adaptation for studying multi-agent coordination under neutrality (Sec.~\ref{sec:qsg}).
    \item \textbf{Drift--selection scaling laws.} We show how population size, communication bandwidth, adaptation rate, and internal uncertainty jointly determine whether consensus is dominated by amplified sampling noise or by systematic bias (Sec.~\ref{sec:mechanism}).
    \item \textbf{Empirical validation.} We test these predictions in QSG simulations and naming-game experiments with LLM populations, finding agreement with the predicted scaling laws across multiple regimes (Sec.~\ref{sec:experiments}).
\end{enumerate}

\section{Modeling: Quantized Simplex Gossip (QSG)}
\label{sec:qsg}

Inspired by naming games, we consider \(N\) agents naming a fixed referent \(r\) using a \(K\)-way label set \cite{Steels1995SpatialVocab,Baronchelli2006Sharp}.
Unlike classical naming games, which use deterministic inventory updates \cite{Steels1995SpatialVocab,Baronchelli2006Sharp,Baronchelli2008NG}, QSG represents each agent's belief as a point on the probability simplex, which lets us analyze sampling noise directly.
Each interaction selects an ordered speaker--listener pair, the speaker emits a message, and the listener records the exchange.
Because these episodes are the population's only experience, internal beliefs evolve through them, and coordination arises through in-context learning rather than explicit reward.

The central modeling challenge is the mismatch between continuous internal beliefs and discrete communication.
Even under neutrality, sampling a discrete message from a continuous distribution injects stochasticity that can destabilize symmetric states and drive consensus.
We capture this with \textbf{Quantized Simplex Gossip (QSG)}. Each agent holds a distribution \(x_i\in\Delta^{K-1}\), the speaker sends a quantized message (Hard/Top-$m$/Soft), and the listener updates toward it with adaptation rate \(\alpha\) \cite{Martins2008CODA}.
Figure~\ref{fig:nnd_protocol} schematizes the QSG interaction protocol, with random speaker--listener pairing, quantized communication, and listener adaptation.

\begin{figure}[t]
    \centering
    \includegraphics[width=0.8\linewidth]{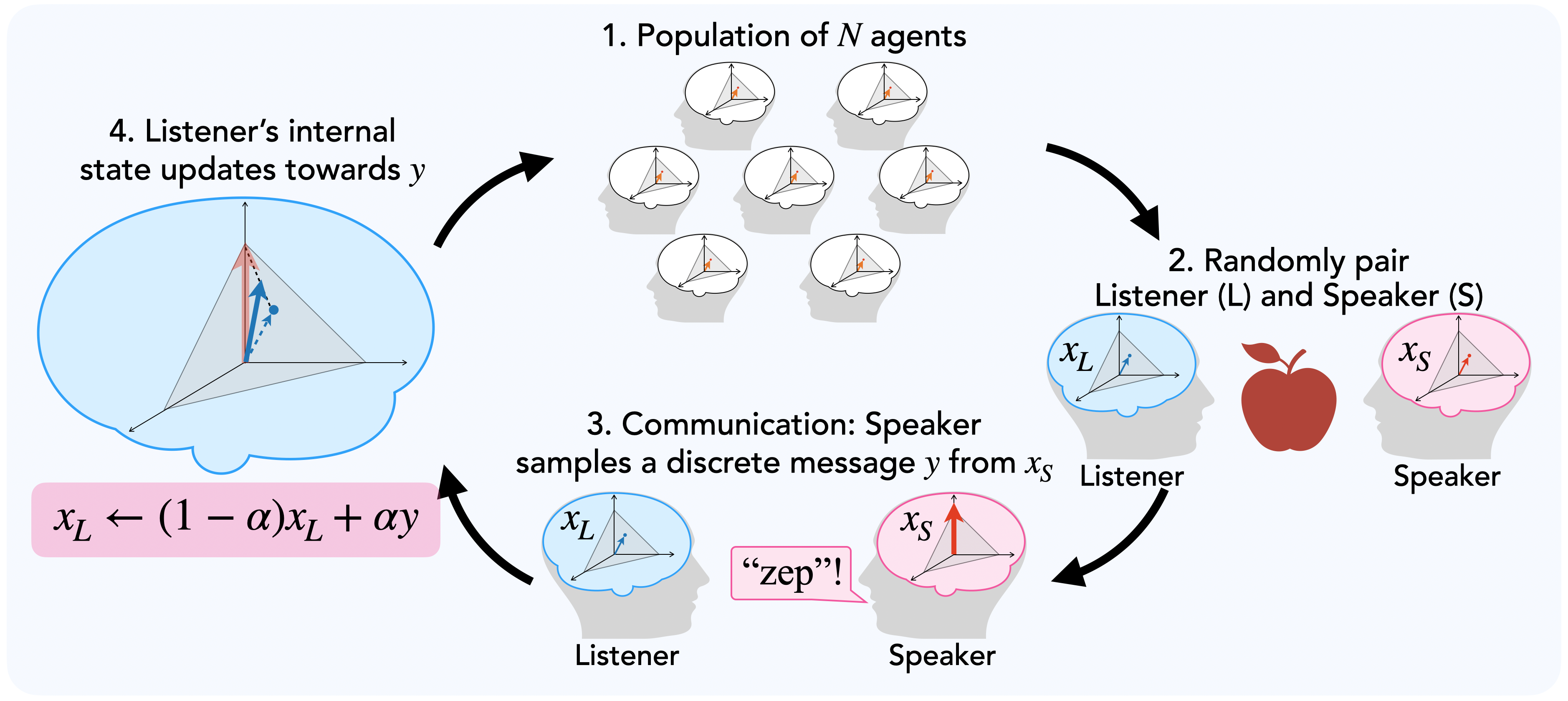}
\caption{\textbf{Quantized Simplex Gossip (QSG): interaction protocol.}
1. A population of \(N\) agents, each with an internal state \(x_i \in \Delta^{K-1}\). 2. An ordered speaker--listener pair is selected uniformly at random. 3. For a fixed referent, the speaker samples and sends a discrete message \(y\) from \(x_S\) in the quantized regimes (Hard or Top-\(m\)); Soft corresponds to transmitting the full distribution. 4. The listener updates toward the received message with adaptation rate \(\alpha\), \(x_L \leftarrow (1-\alpha)x_L + \alpha y\).}
    \label{fig:nnd_protocol}
    \end{figure}

\textbf{Agent state and neutrality.}
Fix \(K\ge2\) conventions and a population of \(N\ge2\) agents.
Agent \(i\in\{1,\dots,N\}\) maintains an internal probability distribution \(x_i\in\Delta^{K-1}\), where
\[
\Delta^{K-1} \triangleq \Bigl\{x\in\mathbb{R}^K:\; x_k\ge 0,\ \sum_{k=1}^K x_k = 1\Bigr\}.
\]
The population state is \(X=(x_1,\dots,x_N)\in(\Delta^{K-1})^N\).
Throughout, token labels are \emph{neutral and exchangeable}: there is no intrinsic fitness difference, external reward signal, or prior bias associated with any convention index \(k\in\{1,\dots,K\}\).

We represent each agent by a continuous distribution \(x_i\) because LLM token generation induces a distribution over discrete tokens, and uncertainty in that distribution is a central control variable.
Communication, however, is tokenized: agents exchange discrete outputs or short summaries, not real-valued internal representations.
This separation between continuous internal state and quantized message is the minimal mechanism that produces endogenous stochastic sampling under neutrality.

\textbf{Quantized communication.}
At each interaction, we sample an \emph{ordered} pair \((S,L)\) of distinct agents uniformly from all \(N(N-1)\) speaker--listener pairs.
The speaker \(S\) samples a message \(y\) from its internal distribution \(x_S\) and transmits it to the listener.
This message is the only communicated object; it is a discrete sample or short list drawn from \(x_S\).
We parameterize the amount of information per interaction by an effective bandwidth \(m\in\{1,2,\dots\}\cup\{\infty\}\).
Let \(\{e_k\}_{k=1}^K\) denote the standard basis in \(\mathbb{R}^K\).
\begin{itemize}
    \item \emph{Hard} (\(m=1\)): sample \(k^\star\sim \mathrm{Cat}(x_S)\), i.e.\ \(\Pr(k^\star = k) = (x_S)_k\), and transmit \(y=e_{k^\star}\).
    \item \emph{Top-\(m\)} (\(m<\infty\)): sample \(k_1,\dots,k_m \overset{\mathrm{iid}}{\sim} \mathrm{Cat}(x_S)\) and transmit the empirical message
    \begin{equation}
    y^{(m)} \;=\; \frac{1}{m}\sum_{j=1}^m e_{k_j}.
    \label{eq:empirical-message}
    \end{equation}
\item \emph{Soft} (\(m=\infty\)): transmit the full distribution \(y=x_S\) deterministically.
\end{itemize}
Under all three regimes, the conditional mean is identical, \(\mathbb{E}[y\mid x_S]=x_S\), while the conditional variance of the message decreases as \(m\) increases (scaling as \(1/m\) in the Top-\(m\) family).
Soft is randomized gossip averaging \cite{Boyd2006Gossip}; Hard resembles a voter/Moran copying process \cite{Moran1958RandomProcesses,CliffordSudbury1973SpatialConflict,HolleyLiggett1975VoterModel,Liggett1999Interacting}.
Top-$m$ is a multi-label / multi-bit variant: messages land on a higher-dimensional face of the simplex rather than a vertex.

The bandwidth parameter \(m\) controls how much evidence the listener receives per interaction: a single discrete choice (one token) or a short transcript providing multiple draws.
Top-\(m\) treats these variants uniformly by controlling message variance while keeping \(\mathbb{E}[y\mid x_S]\) fixed.
Soft (\(m=\infty\)) is an analytic baseline that removes quantization noise; it does not assume literal transmission of hidden states, but serves as an infinite-bandwidth reference against which drift can be identified.

\textbf{Listener update and in-context adaptation.}
After receiving a message \(y\), the listener moves a single step toward it.
Work on in-context learning interprets the forward pass as an implicit online update \cite{Akyurek2022ICL,vonOswald2023ICLGD,Dai2023WhyGPT,Park2025ICLRRepresentations}.
Recent work also suggests that in-context behavior can reflect a mixture of distinct strategies rather than a single monolithic algorithm \cite{Wurgaft2025RationalStrategies}.
Motivated by that view, we model each interaction as a single in-context--style adaptation step with rate \(\alpha\in(0,1]\):
\begin{equation}
x_L' \;=\; (1-\alpha)x_L + \alpha y,
\label{eq:update}
\end{equation}
while the speaker and all other agents remain unchanged.
Equation~\eqref{eq:update} is a minimal abstraction of this viewpoint: the listener performs one adaptation step toward a target distribution encoded by the received message.
The convex-combination form is the simplest contractive update on the simplex that isolates adaptation rate \(\alpha\) while remaining analytically tractable.
In the Soft case, this update reduces to DeGroot-style opinion averaging \cite{DeGroot1974,FriedkinJohnsen1990SocialInfluence} with a simplex-valued opinion $x_i\in\Delta^{K-1}$; in the quantized regimes, it is the corresponding update toward a sampled message rather than a full belief state.
Small \(\alpha\) yields weak per-interaction adaptation (slow accumulation of influence), whereas larger \(\alpha\) amplifies the effect of each sampled message on the listener's belief state.

\paragraph{QSG as a null model: explicit assumptions.}
QSG is defined by the following assumptions, which we later treat as empirically testable approximations:
\begin{enumerate}[leftmargin=1.2em]
\item[A1.] \textbf{Simplex state}
Each agent \(i\) is represented by a probability distribution \(x_i \in \Delta^{K-1}\) over \(K\) competing conventions, corresponding to a fixed naming-game prompt or referent \(r\).
\item[A2.] \textbf{Well-mixed pair selection}
At each step, an ordered speaker--listener pair \((S,L)\) is sampled uniformly from the \(N(N-1)\) pairs with \(S \neq L\).
\item[A3.] \textbf{Quantized message channel}
The speaker communicates a finite-bandwidth message obtained by sampling from \(x_S\) (Hard or Top-\(m\)), rather than transmitting the full distribution.
\item[A4.] \textbf{First-order adaptation} (motivated by in-context learning)
Each interaction induces a single adaptation step of the listener distribution toward the received message, with magnitude controlled by a scalar adaptation rate \(\alpha\).
\end{enumerate}

\textbf{Control parameters and macroscopic observables.}
The core control parameters are \((N,\alpha,m)\), with label count $K$ fixed per experiment.
For later use in the analysis, we also track the population mean \(\bar{x}\), the polarization \(U \coloneqq \|\bar{x}\|_2^2\), and the disagreement energy \(V \coloneqq \sum_i\|x_i-\bar{x}\|_2^2\).
Table~\ref{tab:qsg-controls} summarizes this notation.

\begin{table}[t]
\centering
\small
\setlength{\tabcolsep}{6pt}
\begin{minipage}[t]{0.48\linewidth}
\centering
\textbf{(a) Control parameters}\\[4pt]
\begin{tabular}{@{}ll@{}}
\toprule
Symbol & Meaning \\
\midrule
\(N\) & population size \\
\(\alpha\in(0,1]\) & adaptation rate \\
\(m\in\{1,\dots,\infty\}\) & communication bandwidth \\
\bottomrule
\end{tabular}
\end{minipage}%
\hfill
\begin{minipage}[t]{0.48\linewidth}
\centering
\textbf{(b) Macroscopic observables}\\[4pt]
\begin{tabular}{@{}ll@{}}
\toprule
Symbol & Meaning \\
\midrule
\(\bar{x}\in\Delta^{K-1}\) & population mean \\
\(U=\|\bar{x}\|_2^2\) & polarization \\
\(V=\sum_i\|x_i{-}\bar{x}\|_2^2\) & disagreement energy \\
\bottomrule
\end{tabular}
\end{minipage}
\caption{\textbf{QSG control parameters and macroscopic observables.} Each agent holds an internal state $x_i \in \Delta^{K-1}$ over $K$ labels; the left panel lists the control parameters, and the right panel lists the macroscopic observables used in the drift analysis (Sec.~\ref{sec:mechanism}).}
\label{tab:qsg-controls}
\end{table}

\section{Analysis and Scaling Laws}
\label{sec:mechanism}

We analyze QSG at the population level to ask when coordination is shaped mainly by selection versus stochastic drift, using macroscopic observables defined directly from the population state.
The drift mechanism we isolate is not mutually exclusive with selection, but adds to any systematic biases.

\subsection{Macroscopic observables}
We analyze the QSG dynamics directly on the simplex using the update rules in Sec.~\ref{sec:qsg} (Eqs.~\eqref{eq:empirical-message} and \eqref{eq:update}).
This lets us track how sampling variance perturbs the mean dynamics near symmetry.
The primary order parameter is the population mean distribution, analogous to magnetization in statistical mechanics.
\begin{equation}
\bar{x} \coloneqq \frac{1}{N}\sum_{i=1}^N x_i \in \simplex^{K-1}.
\end{equation}
From $\bar{x}$ we define the \textbf{polarization} (squared magnetization).
\begin{equation}
U \coloneqq \|\bar{x}\|_2^2 \in \left[\tfrac{1}{K}, 1\right],
\label{eq:U_def}
\end{equation}
which equals $1/K$ at perfect symmetry and $1$ at consensus.
It measures collective alignment regardless of which label ultimately wins.
We also track a diagnostic for agent heterogeneity, the \textbf{disagreement energy}.
\begin{equation}
V \coloneqq \sum_{i=1}^N\|x_i-\bar{x}\|_2^2 \;\ge\; 0.
\label{eq:V_def}
\end{equation}
It measures how dispersed agents remain around the population mean.
At the perfectly symmetric initialization, sampling noise is the only driver of motion; more generally, drift and selection can coexist, and these observables isolate the drift component.
Figure~\ref{fig:concept} visualizes the simplex geometry and the sampling-noise--driven drift strength that underlies the analysis.

\begin{figure}[t]
\centering
\begin{minipage}[t]{0.30\linewidth}
\centering
\includegraphics[width=\linewidth]{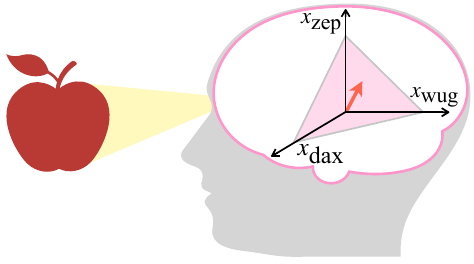}
\par\smallskip
\small (a) Probabilistic naming.
\end{minipage}\hfill
\begin{minipage}[t]{0.35\linewidth}
\centering
\includegraphics[width=\linewidth]{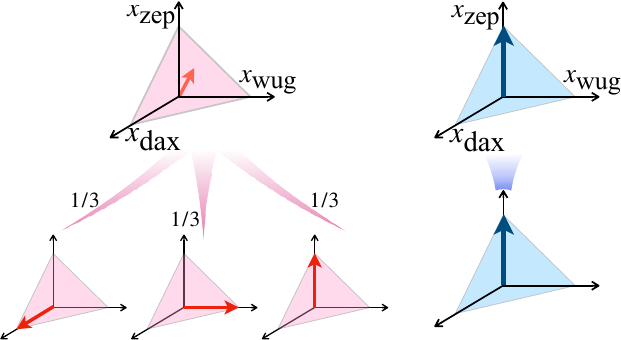}
\par\smallskip
\small (b) Sampling noise from quantization.
\end{minipage}\hfill
\begin{minipage}[t]{0.30\linewidth}
\centering
\includegraphics[width=\linewidth]{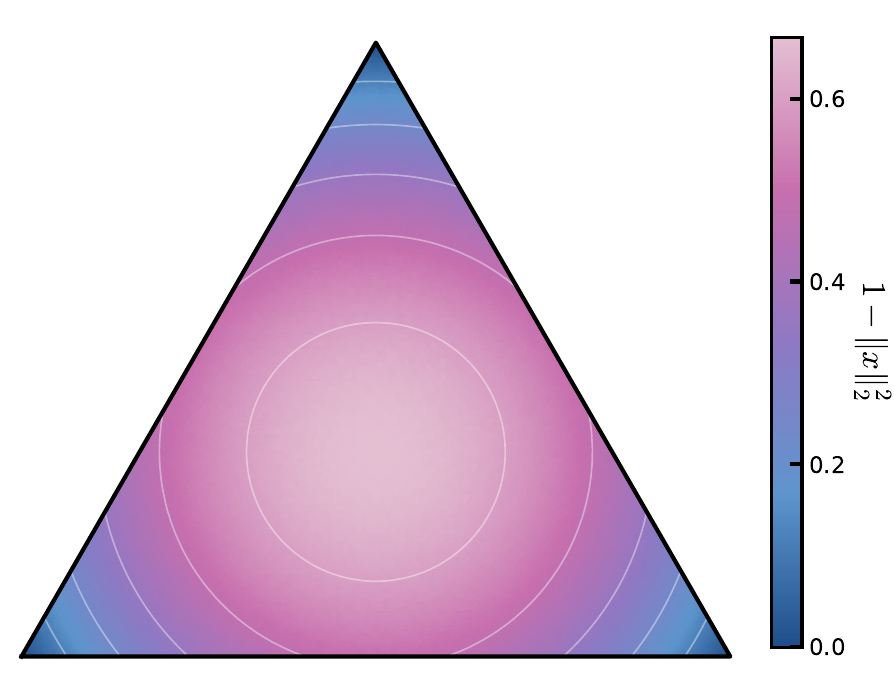}
\par\smallskip
\small (c) Drift strength on the simplex.
\end{minipage}
\caption{\textbf{Probabilistic naming and sampling-noise geometry.}
(a) A referent is represented internally as a distribution over candidate labels.
(b) Agent states lie on the simplex: near-uniform, high-entropy states generate larger sampling noise under quantization, whereas peaked, low-entropy states generate less.
(c) Sampling-driven drift strength on the simplex, proportional to $1-\|x\|_2^2$, is maximal near the center and vanishes at the vertices.}
\label{fig:concept}
\end{figure}

\subsection{Soft exchange preserves the mean in expectation and contracts disagreement}
In Soft QSG, the speaker transmits $y=x_S$ and only the listener updates,
$x_L'=(1-\alpha)x_L+\alpha x_S$.
Thus the population mean $\bar{x}\coloneqq \frac{1}{N}\sum_{i=1}^N x_i$ evolves as
\begin{equation}
\bar{x}'=\bar{x}+\frac{\alpha}{N}(x_S-x_L).
\label{eq:mean_update_soft}
\end{equation}
Conditioned on the current state $X=(x_1,\dots,x_N)$ and the uniform random choice of ordered pair $(S,L)$ with $S\neq L$,
\begin{equation}
\mathbb{E}[\bar{x}'\mid X]=\bar{x},
\end{equation}
so $\bar{x}$ is preserved in expectation (equivalently, each coordinate $(\bar{x}_k(t))$ is a bounded martingale), but it is \emph{not} generally invariant along individual trajectories.
Moreover, the disagreement energy $V\coloneqq \sum_{i=1}^N\|x_i-\bar{x}\|_2^2$ contracts in expectation:
\begin{equation}
\mathbb{E}[\Delta V\mid X]
=
-\frac{2\alpha}{N-1}\left(1-\alpha+\frac{\alpha}{N}\right)V
\;\le\;0.
\label{eq:soft_V_contraction}
\end{equation}
Hence, if $x_i(0)=\mathbf{1}/K$ for all $i$, then $x_i(t)=\mathbf{1}/K$ for all $t$ (no symmetry breaking).
This provides the neutral baseline. Full-distribution exchange smooths disagreement but does not create spontaneous convention formation under neutrality.
Against this baseline, any symmetry breaking in the quantized regimes must come from the extra sampling variance injected by communication itself.

\subsection{Hard sampling injects an extra variance term}
Hard and Soft share the \emph{same conditional mean} (since $\mathbb{E}[e_{k^\star}\mid x_S]=x_S$), but Hard injects additional sampling variance.

\begin{theorem}[Hard sampling increases polarization via sampling variance]
\label{thm:hard_variance}
Consider QSG with adaptation rate $\alpha\in(0,1]$.
Let $\bar{x}$ be the population mean and define the polarization potential $U\coloneqq \|\bar{x}\|_2^2$.
Conditioned on the current state $X$, the expected one-step change in $U$ satisfies
\begin{align}
\mathbb{E}[\Delta U \mid X]_{\mathrm{hard}}
&= \mathbb{E}[\Delta U \mid X]_{\mathrm{soft}} \nonumber\\
&\quad + \frac{\alpha^2}{N^2}\,\mathbb{E}\big[1-\|x_S\|_2^2 \mid X\big], \label{eq:extra_variance}
\end{align}
where the expectation is over the random choice of $(S,L)$ and (for Hard) the sample $k^\star$.
The additional term in \eqref{eq:extra_variance} is the sampling variance and is strictly positive whenever $x_S$ is not one-hot.
In particular, at the perfectly symmetric initialization $x_i=\mathbf{1}/K$, we have $\mathbb{E}[\Delta U \mid X]_{\mathrm{soft}}=0$ but
\begin{equation}
\mathbb{E}[\Delta U \mid X]_{\mathrm{hard}} = \frac{\alpha^2}{N^2}\left(1-\frac{1}{K}\right) > 0,
\end{equation}
so symmetry is noise-unstable under Hard sampling.
\end{theorem}

\noindent\textit{Proof.} See Appendix~\ref{app:theorem1_proof}.

This variance injection drives symmetry breaking in the neutral model under symmetric initialization and adds to any selection effects.
To see what Hard adds beyond Soft, note that under Soft exchange the symmetric state $\bar{x}=\mathbf{1}/K$ is a fixed point and drift arises only from heterogeneity ($x_S\neq x_L$), so $U$ grows only insofar as agents disagree.
Hard/Top-$m$ adds variance even when $x_S=x_L$, so the symmetric interior state becomes stochastically unstable and trajectories are pushed toward consensus.
Equivalently, $\mathbb{E}[\Delta U \mid X]_{\mathrm{soft}}=\frac{2\alpha^2}{N^2(N-1)}V$, while Hard adds the positive term in Eq.~\eqref{eq:extra_variance}.

\paragraph{Mean-field approximation and consensus time.}
Under a mean-field ansatz where agents remain similar ($x_i \approx p$ for all $i$), we have $\|x_S\|_2^2 \approx U$, yielding the mean-field ODE (approximating the expected trajectory)
\begin{equation}
\frac{dU}{dt} = \frac{\alpha^2}{mN^2}(1-U).
\label{eq:meanfield_ode}
\end{equation}
Starting from the symmetric state $U(0)=1/K$, the solution is
\begin{equation}
U(t) = 1 - \left(1-\frac{1}{K}\right)\exp\!\left(-\frac{\alpha^2 t}{mN^2}\right),
\label{eq:meanfield_solution}
\end{equation}
with characteristic time $t_{\mathrm{char}}\sim mN^2/\alpha^2$ interaction steps, or equivalently $\tau_{\mathrm{char}}\sim mN/\alpha^2$ population rounds ($\tau=t/N$), with Hard corresponding to $m=1$.
Mean-field predicts an approximate single-parameter collapse of $U(t)$ across $N$ when time is rescaled by $mN^2/\alpha^2$.
For a fixed consensus threshold $U_\star\in(1/K,1)$, the mean-field hitting time satisfies
\begin{equation}
t_{\mathrm{cons}}(U_\star)
\;\approx\;
\frac{mN^2}{\alpha^2}\log\!\left(\frac{1-1/K}{1-U_\star}\right).
\label{eq:tcons_scaling}
\end{equation}
The logarithmic factor depends only on $(K,U_\star)$, so $t_{\mathrm{cons}}\propto N^2$ in steps (equivalently $\tau_{\mathrm{cons}}=t_{\mathrm{cons}}/N \propto N$ in population rounds).
Physically, larger populations and higher-bandwidth messages weaken the stochastic push of any single interaction, whereas stronger in-context adaptation accelerates the approach to consensus.
Figure~\ref{fig:meanfield_validation} compares simulated Hard trajectories to the mean-field curve; trajectories track the mean-field prediction, consistent with the variance-injection mechanism in Theorem~\ref{thm:hard_variance}.

\begin{figure*}[t]
\centering
\begin{subfigure}[t]{0.30\linewidth}
\centering
\includegraphics[width=\linewidth]{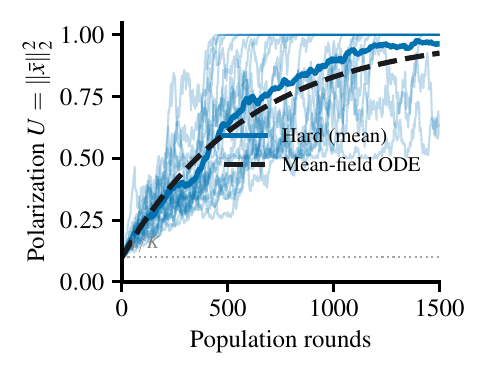}
\caption{Mean-field trajectories.}
\label{fig:meanfield_validation}
\end{subfigure}
\hfill
\begin{subfigure}[t]{0.34\linewidth}
\centering
\includegraphics[width=\linewidth]{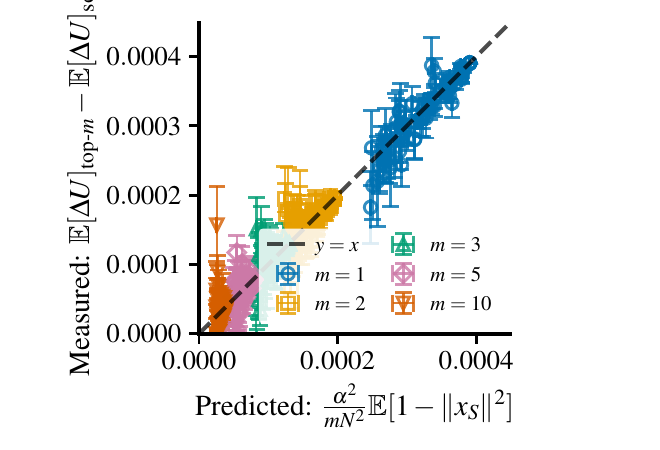}
\caption{Measured vs.\ predicted drift.}
\label{fig:thm2_validation}
\end{subfigure}
\hfill
\begin{subfigure}[t]{0.30\linewidth}
\centering
\includegraphics[width=\linewidth]{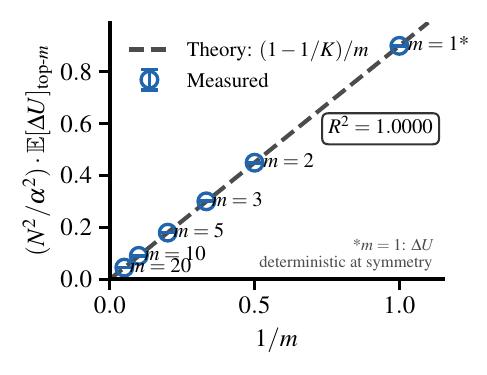}
\caption{$1/m$ scaling at symmetry.}
\end{subfigure}
\caption{\textbf{QSG simulations validate the mean-field approximation and drift identities.}
(a) Hard-sampling trajectories and their ensemble mean track the mean-field solution in Eq.~\eqref{eq:meanfield_solution} ($N=24$, $K=10$, $\alpha=0.2$).
(b) One-step excess drift from shared snapshot states $X$, plotted against the predicted variance-injection term; dashed reference line $y=x$ ($N=24$, $K=10$, $\alpha=0.5$).
(c) Symmetric $1/m$ corollary. Under $x_i=\mathbf{1}/K$, the normalized excess drift follows the theoretical line $(1-1/K)(1/m)$ ($N=24$, $K=10$, $\alpha=0.5$).}
\label{fig:qsg_simulations}
\end{figure*}

\subsection{Top-$m$ reduces the symmetry-breaking drift as $1/m$}
To quantify bandwidth, we model the candidate list as $m$ i.i.d.\ samples from $x_S$ and transmit their empirical distribution.
This preserves the mean but reduces sampling variance by $1/m$.

The variance of the empirical message follows from the Top-$m$ construction. Let
$k_1,\dots,k_m \overset{\mathrm{iid}}{\sim} \mathrm{Cat}(x_S)$ and define
$y^{(m)} \coloneqq \frac{1}{m}\sum_{j=1}^m e_{k_j}$.
Then $\mathbb{E}[y^{(m)}\mid x_S]=x_S$ and
\begin{equation}
\mathrm{Cov}\!\left(y^{(m)} \mid x_S\right) = \frac{1}{m}\left(\mathrm{diag}(x_S) - x_S x_S^\top\right).
\end{equation}
In particular, $\mathbb{E}\big[\|y^{(m)}-x_S\|_2^2 \mid x_S\big] = \frac{1}{m}\left(1-\|x_S\|_2^2\right)$, i.e., variance scales as $1/m$.

\begin{theorem}[Top-$m$ drift term scales as $1/m$]
\label{thm:topm_scaling}
Consequently, under QSG with Top-$m$ empirical communication $y^{(m)}$, the polarization drift satisfies
\begin{equation}
\begin{aligned}
\mathbb{E}[\Delta U \mid X]_{\mathrm{topm}}
&= \mathbb{E}[\Delta U \mid X]_{\mathrm{soft}}\\
&\quad + \frac{\alpha^2}{m\,N^2}\,\mathbb{E}\big[1-\|x_S\|_2^2 \mid X\big].
\end{aligned}
\label{eq:topm_drift}
\end{equation}
At the perfectly symmetric initialization $x_i=\mathbf{1}/K$, this yields
$\mathbb{E}[\Delta U \mid X]_{\mathrm{topm}} = \frac{\alpha^2}{mN^2}\left(1-\frac{1}{K}\right)$.
\end{theorem}

Thus, increasing bandwidth weakens symmetry-breaking drift linearly in $1/m$.
\noindent\textit{Proof.} See Appendix~\ref{app:theorem2_proof}.

\begin{wrapfigure}[15]{r}{0.40\linewidth}
\vspace{-1.5\baselineskip}
\centering
\includegraphics[width=0.86\linewidth]{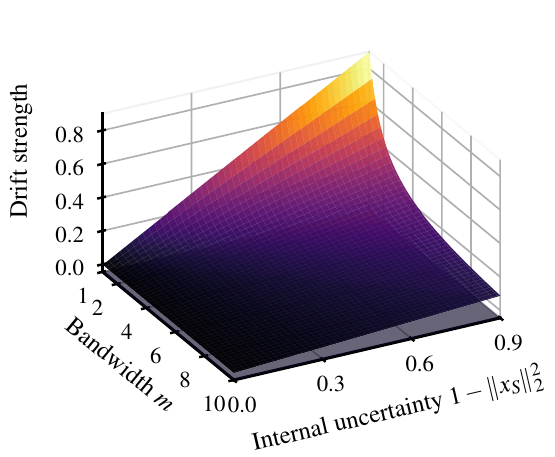}
\caption{\textbf{Sampling-driven drift from uncertainty and bandwidth.}
For fixed $\alpha$ and $N$, the drift term in Eq.~\eqref{eq:topm_drift} scales as $(1-\|x_S\|_2^2)/m$, increasing with uncertainty and decreasing with bandwidth $m$.}
\label{fig:entropic_force}
%\vspace{-0.15\baselineskip}
\end{wrapfigure}

Figure~\ref{fig:entropic_force} visualizes how the drift term depends on uncertainty and bandwidth, and Figure~\ref{fig:thm2_validation} tests Eq.~\eqref{eq:topm_drift} directly by estimating one-step conditional expectations from shared snapshot states and comparing the measured excess drift to the predicted variance-injection term (Appendix~\ref{app:experiments} summarizes the estimator and parameters). 
Panel (c) of Fig.~\ref{fig:qsg_simulations} shows the corresponding $1/m$ scaling at the symmetric initialization.

\subsection{Absorption at $\alpha=1$ and run-level symmetry breaking for $\alpha<1$}
For $\alpha=1$ (pure copying), after each agent has served as a listener at least once, all $x_i$ are one-hot and Hard QSG reduces to a finite-state copying Markov chain with absorbing consensus states $\{x_i=e_k,\ \forall i\}$.
For $\alpha<1$ the state space is continuous, and we do not claim finite-time consensus or almost-sure absorption. Our formal results in this regime are expectation-level drift identities and mean-field approximations. In simulations, the process tends to polarize and empirically select a winner in each run even though ensemble averages preserve symmetry.
This is the discrete analogue of the drift mechanism above. Randomness can still select a winner even when the ensemble is neutral, but for $\alpha<1$ we treat this as an empirical run-level phenomenon rather than a proved absorption result.
For $\alpha<1$, Eq.~\eqref{eq:logistic_tanh_biased} should be read as a diffusion-approximation description of run-level fixation behavior.
In the hard-copying limit on a complete interaction graph, this reduction becomes a classical voter-model statement \cite{CliffordSudbury1973SpatialConflict,HolleyLiggett1975VoterModel}. The process almost surely reaches a consensus vertex in finite time. If the initialization is exchangeable across token labels, for example $x_i(0)=\mathbf{1}/K$, the winning token is uniformly distributed over $\{1,\dots,K\}$. More generally, the probability that token $k$ wins equals $\bar{x}_k(0)$, by the martingale property of the density process. Appendix~\ref{app:winner_symmetry_proof} gives the reduction and proof details.
This neutral winner symmetry is the reference point that the weak asymmetries below perturb into the drift--selection crossover.

\subsection{Weak asymmetry and the drift--selection crossover ($K=2$)}
\label{sec:weak_asymmetry}

To model weak asymmetry, we tilt only the speaker channel by a small sampling bias (external field) $h$.
For $K=2$, let $p_i:=x_{i1}$ and sample messages from $\tilde p_S=\frac{p_S e^{h}}{p_S e^{h}+(1-p_S)}$ (listener update unchanged).
A diffusion approximation (Appendix~\ref{app:weak_asymmetry_diffusion}) collapses fixation statistics onto the single parameter $\Gamma_h\equiv \frac{mN}{\alpha}h$.
This parameter measures the competition between systematic bias and endogenous drift. Larger $N$ or $m$ suppress the neutral sampling noise and make the same bias more decisive, whereas larger $\alpha$ strengthens drift relative to that same $h$.
Defining the final magnetization $M_\infty \coloneqq 2\bar{x}_1(\infty)-1$,
\begin{equation}
\begin{aligned}
\Pr(\text{label 1 fixes})&\approx\frac{1}{1+\exp(-\Gamma_h)},\\
\mathbb{E}[M_\infty]&\approx 2\Pr(\text{label 1 fixes})-1,\\
N_c&\sim \frac{\alpha}{m|h|}.
\end{aligned}
\label{eq:logistic_tanh_biased}
\end{equation}
Thus fixation is approximately logistic in $\Gamma_h$, with crossover scale $N_c\sim \alpha/(m|h|)$ (from $|\Gamma_h|\sim 1$).
This crossover is previewed in Fig.~\ref{fig:luckiest_vs_fittest}.
We use $\Gamma_h$ to distinguish this bias-based crossover from the temperature-based $\Gamma_T$ in Appendix~\ref{app:Gamma}.
Consequently, $|\Gamma_h|\ll 1$ yields near-neutral winners (near $1/2$), while $|\Gamma_h|\gg 1$ yields bias-driven amplification of the asymmetry.

These results point to a single mechanism linking population size, bandwidth, adaptation strength, and internal uncertainty. Quantized communication injects sampling variance, and that variance controls both the speed of neutral consensus formation and the extent to which weak biases are amplified. That mechanism yields testable scalings, including $1/N^2$ and $1/m$ early drift and a mean-field consensus time, which we evaluate next in LLM populations.

\section{Experimental Validation}
\label{sec:experiments}

We test the QSG scaling predictions in LLM populations using the Neutral Naming Drift (NND) protocol. In each run, $N$ agents repeatedly name a fixed referent $r$ using $K$ neutral synthetic labels, with no external reward or ground truth.
Interactions are ordered speaker--listener updates. In the $N$-sweep, the speaker emits one label, whereas in the separate Top-$m$ sweep the speaker emits exactly $m$ labels. In both cases only the listener updates, so agents respond from their own memory rather than from a partner's current output.

Empirical observables are estimated from probe outputs that are used only for measurement and are never incorporated into memory.
At each probe time, we estimate the population-average label distribution $p$ from sampled agent outputs, yielding the empirical counterpart of the theoretical population mean $\bar{x}$ defined in Sec.~\ref{sec:mechanism}. In the Top-$m$ sweep, repeated probe draws support a direct estimate of $U$, whereas in the $N$-sweep the plotted squared-frequency statistic is a finite-sample proxy. We therefore focus on ensemble scaling trends rather than exact pointwise agreement.
We estimate early drift as $\Delta U/\text{step}$, and for each trial define time to consensus as the first probe where $U \ge U_\star$ (here $U_\star=0.9$). For scaling plots, we aggregate this quantity over trials that also finish above the same threshold at the run horizon.
Because $p$ is estimated from finite samples, $U$ is itself noisy, so we report cross-trial variability rather than over-interpreting individual runs.
Appendix~\ref{app:llm_population} details the probe cadence, sampling budget, decoding controls, and prompt templates.
Figure~\ref{fig:theory_vs_llm} summarizes the LLM scaling tests.

\begin{figure}[t]
\centering
\begin{subfigure}[t]{0.323\linewidth}
\centering
\includegraphics[width=\linewidth]{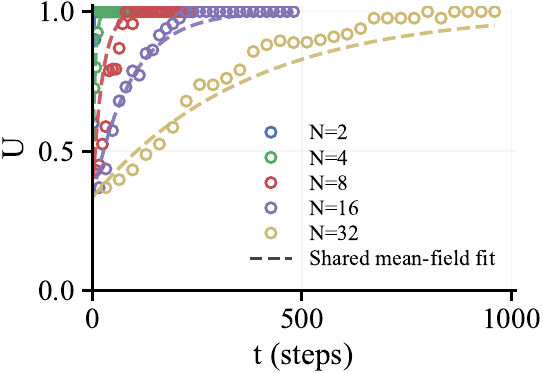}
\caption{GPT-4o $U(t)$.}
\label{fig:theory_vs_llm_gpt_traj}
\end{subfigure}
\hfill
\begin{subfigure}[t]{0.323\linewidth}
\centering
\includegraphics[width=\linewidth]{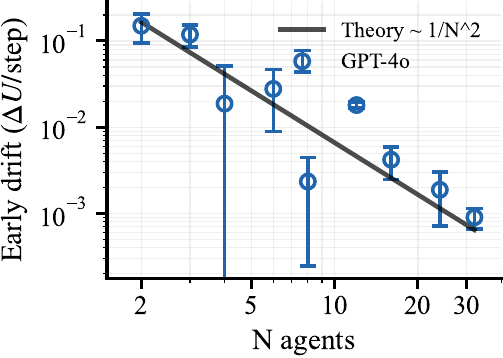}
\caption{GPT-4o drift.}
\label{fig:theory_vs_llm_gpt_drift}
\end{subfigure}
\hfill
\begin{subfigure}[t]{0.323\linewidth}
\centering
\includegraphics[width=\linewidth]{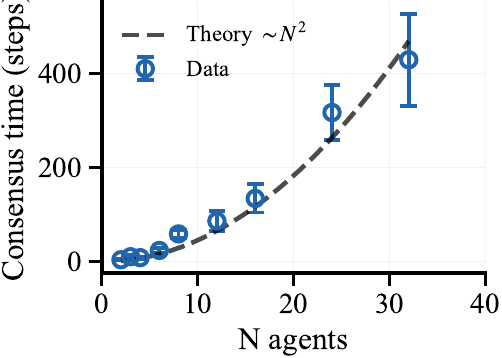}
\caption{GPT-4o $t_{\mathrm{cons}}$.}
\label{fig:theory_vs_llm_gpt_tcons}
\end{subfigure}
\par\vspace{2pt}
\begin{subfigure}[t]{0.323\linewidth}
\centering
\includegraphics[width=\linewidth]{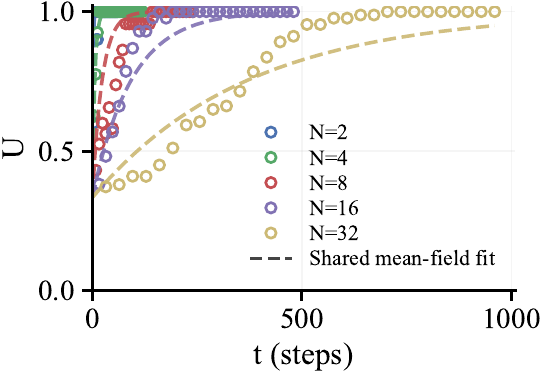}
\caption{Claude Haiku 4.5 $U(t)$.}
\label{fig:theory_vs_llm_claude_traj}
\end{subfigure}
\hfill
\begin{subfigure}[t]{0.323\linewidth}
\centering
\includegraphics[width=\linewidth]{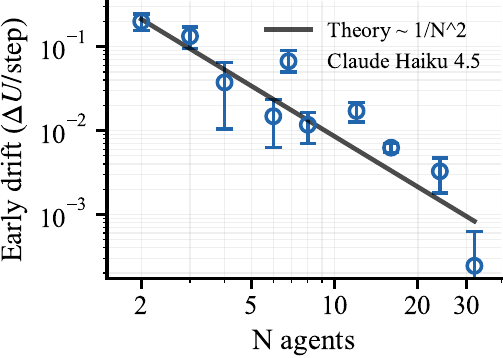}
\caption{Claude Haiku 4.5 drift.}
\label{fig:theory_vs_llm_claude_drift}
\end{subfigure}
\hfill
\begin{subfigure}[t]{0.323\linewidth}
\centering
\includegraphics[width=\linewidth]{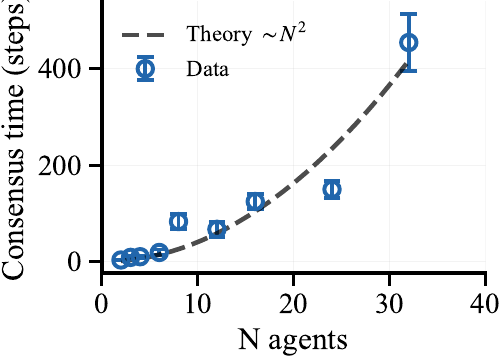}
\caption{Claude Haiku 4.5 $t_{\mathrm{cons}}$.}
\label{fig:theory_vs_llm_claude_tcons}
\end{subfigure}
\caption{\textbf{LLM scaling tests of QSG predictions.}
Top row: GPT-4o; bottom row: Claude Haiku 4.5.
Panels (a,d) show $K=3$ polarization trajectories well captured by a single effective-\(\alpha\) mean-field fit shared across $N$; panels (b,e) show early drift $\Delta U/\mathrm{step}$ decreasing near the predicted $1/N^2$ law; and panels (c,f) show time to consensus increasing close to dashed $cN^2$ fits.}
\label{fig:theory_vs_llm}
\end{figure}

Figure~\ref{fig:theory_vs_llm} shows strong scaling-level agreement between the QSG predictions and the LLM experiments in both GPT-4o and Claude Haiku 4.5.
Panels~\subref{fig:theory_vs_llm_gpt_traj} and \subref{fig:theory_vs_llm_claude_traj} show that the raw (non-rescaled) $U(t)$ trajectories for $K=3$ are captured well by the mean-field form
$U(t)=1-(1-1/K)\exp(-\alpha^2 t/(mN^2))$ when fit with a \emph{single} effective $\alpha$ shared across $N$.
Panels~\subref{fig:theory_vs_llm_gpt_drift} and \subref{fig:theory_vs_llm_claude_drift} show that the early-time drift decreases with $N$ close to the predicted $1/N^2$ law from the variance-injection term in Theorem~\ref{thm:hard_variance}. Panels~\subref{fig:theory_vs_llm_gpt_tcons} and \subref{fig:theory_vs_llm_claude_tcons} show that time to consensus in total interactions (steps) grows close to quadratically in $N$, again matching the mean-field prediction $t_{\mathrm{cons}}\sim mN^2/\alpha^2$ (Eq.~\eqref{eq:tcons_scaling}). In population rounds $\tau=t/N$, the same scaling is linear in $N$.
Taken together, the same variance-injection mechanism that organizes the QSG theory accounts for the main empirical patterns in both LLM families.

For the Top-$m$ setting, the LLM protocol mirrors Eq.~\eqref{eq:empirical-message}.
\par\noindent
\begin{wrapfigure}[16]{r}{0.39\linewidth}
\vspace{-1.1\baselineskip}
\centering
\includegraphics[width=0.94\linewidth]{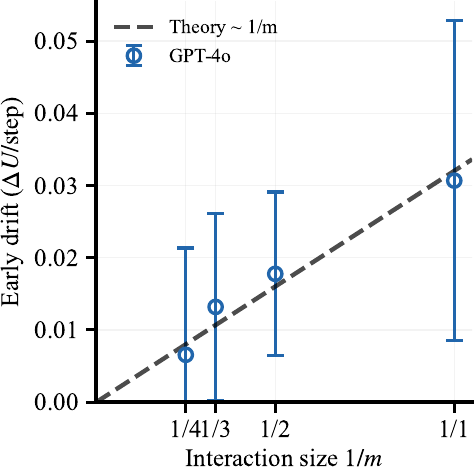}
\caption{\textbf{Top-$m$ scaling in GPT-4o.}
Early-round drift follows the predicted $1/m$ law.}
\label{fig:m_sweep_llm_scaling}
\end{wrapfigure}
The speaker emits $m$ labels, repeats are allowed, and the transmitted list is treated as an unranked multiset; Appendix~\ref{app:topm_llm_details} gives the full protocol.
Figure~\ref{fig:m_sweep_llm_scaling} shows that in GPT-4o the Top-$m$ sweep follows the predicted $1/m$ trend. Across the $N$-sweep and the Top-$m$ sweep, and across two LLM families, QSG captures the main trends in the coordination dynamics.

\FloatBarrier
\section{Discussion and Conclusion}
QSG is deliberately minimal, with the aim of isolating the essential dynamics of mutual in-context learning at the population level rather than reproducing every practical detail of interacting LLM populations. Despite its simplicity, the model captures how quantized communication injects sampling noise and how repeated interaction can amplify that noise into convention formation. This picture leads to the insight that consensus can arise even when no individual agent prefers any output \emph{a priori}, while weak asymmetries can be collectively amplified enough to overcome drift once the system crosses into the selection regime. Agreement in an LLM population is therefore not, on its own, evidence of collective reasoning or information aggregation; it can also be the consequence of memetic drift.

More broadly, the paper also suggests how multi-agent LLM systems can be mechanistically understood in ways that matter for alignment. There is by now a large empirical literature on multi-agent LLM systems, but most of it characterizes collective behavior without theoretically isolating the mechanisms that produce it. By pairing a synthetic task such as the naming game with a minimal analytically tractable model, we can move beyond description and derive scaling laws from first principles, state them as quantitative predictions, and test them in controlled LLM experiments. The scaling laws for $N$, $m$, and $\alpha$ reported here, and the drift--selection crossover they predict, follow from the variance-injection mechanism in QSG and are broadly consistent with both simulations and experiments with LLM populations. This result suggests that synthetic games combined with minimal models can open a productive route toward physics-style analysis in multi-agent LLM systems.

Within that broader methodological view, the framework also suggests a population-level counterpart to mechanistic interpretability. Instead of asking only how a single model forms and uses internal representations through training or in-context learning, we can ask how collective representations emerge, stabilize, and reorganize through social interaction, whether by converging, fragmenting, polarizing, or otherwise becoming distorted. From a safety perspective, this also means that harmful collective representations may form through interaction, echoing concerns about representation formation and bias that are already familiar from work on single models. One possible failure mode is the population-level amplification of sycophancy, in which strategically injected or socially reinforced signals steer a group toward distorted conventions. This in turn raises a broader question for future work about whether alignment at the level of individual agents composes under social interaction, or whether a society of individually aligned agents can still produce misaligned collective outcomes. Addressing that question will require extending the analysis to structured networks, heterogeneous agents, and training-data priors that act as selection forces, so that we can better understand what a population converges on and why, not just whether it coordinates.

QSG, the minimal theory developed here, should be understood in the spirit of the ideal gas law in thermodynamics, setting a bare-bones baseline for multi-agent coordination. More broadly, we see this work as a step toward a physics of social representation formation in interacting model populations. Richer theories can then build on minimal models such as QSG toward a statistical mechanics of interacting agents.

\bibliographystyle{unsrtnat}
\bibliography{references}

%%%%%%%%%%%%%%%%%%%%%%%%%%%%%%%%%%%%%%%%%%%%%%%%%%%%%%%%%%%%%%%%%%%%%%%%%%%%
% SUPPLEMENTARY MATERIAL: DETAILED EXPERIMENT DESCRIPTION
%%%%%%%%%%%%%%%%%%%%%%%%%%%%%%%%%%%%%%%%%%%%%%%%%%%%%%%%%%%%%%%%%%%%%%%%%%%%
\clearpage
\appendix
%%%%%%%%%%%%%%%%%%%%%%%%%%%%%%%%%%%%%%%%%%%%%%%%%%%%%%%%%%%%%%%%%%%%%%%%%%%%
% APPENDIX A: ADDITIONAL THEORY DETAILS
%%%%%%%%%%%%%%%%%%%%%%%%%%%%%%%%%%%%%%%%%%%%%%%%%%%%%%%%%%%%%%%%%%%%%%%%%%%%
\section{Additional Theory for Quantized Simplex Gossip (QSG)}
\label{app:qsg_theory}

This appendix collects algebraic identities and full drift calculations used in Sec.~\ref{sec:mechanism}.

\subsection{Proof of Theorem~\ref{thm:hard_variance}: hard-sampling drift decomposition}
\label{app:theorem1_proof}

Let only the listener update: $x_L'=(1-\alpha)x_L+\alpha y$, where $y=x_S$ (Soft) or $y=e_{k^\star}$ (Hard).
Then $\bar{x}' = \bar{x} + \frac{\alpha}{N}(y-x_L)$.
Expanding $U'=\|\bar{x}'\|_2^2$ gives
$\Delta U=\frac{2\alpha}{N}\langle \bar{x}, y-x_L\rangle + \frac{\alpha^2}{N^2}\|y-x_L\|_2^2$.
Taking conditional expectations yields
$\mathbb{E}[\Delta U \mid X]=\frac{\alpha^2}{N^2}\mathbb{E}\big[\|y-x_L\|_2^2 \mid X\big]$
because unbiased messaging and uniform ordered-pair sampling give
$\mathbb{E}[y\mid X]=\bar{x}=\mathbb{E}[x_L\mid X]$.
For Hard, decompose $y-x_L=(y-x_S)+(x_S-x_L)$.
Since $\mathbb{E}[y-x_S\mid x_S]=0$, the cross term vanishes after conditioning on $(x_S,x_L)$, and
$\mathbb{E}\big[\langle y-x_S,x_S-x_L\rangle \mid x_S,x_L\big]=0$.
Hence
\[
\mathbb{E}\big[\|y-x_L\|_2^2 \mid X\big]
=
\mathbb{E}\big[\|x_S-x_L\|_2^2 \mid X\big]
+
\mathbb{E}\big[\|y-x_S\|_2^2 \mid X\big].
\]
For one-hot $y=e_{k^\star}$ sampled from $x_S$, $\mathbb{E}\|y-x_S\|_2^2 = 1-\|x_S\|_2^2$.
Substituting these two terms gives the Hard variance-injection formula in Theorem~\ref{thm:hard_variance}.

\subsection[Proof of Top-m scaling theorem]{Proof of Theorem~\ref{thm:topm_scaling}: Top-$m$ drift decomposition and scaling}
\label{app:theorem2_proof}

The same expansion applies with $y$ replaced by the Top-$m$ empirical message $y^{(m)}$.
We have $x_L'=(1-\alpha)x_L+\alpha y^{(m)}$ and hence
$\bar{x}' = \bar{x} + \frac{\alpha}{N}(y^{(m)}-x_L)$, so expanding $U'=\|\bar{x}'\|_2^2$ gives
$\Delta U=\frac{2\alpha}{N}\langle \bar{x}, y^{(m)}-x_L\rangle + \frac{\alpha^2}{N^2}\|y^{(m)}-x_L\|_2^2$.
Taking conditional expectations yields
$\mathbb{E}[\Delta U \mid X]=\frac{\alpha^2}{N^2}\mathbb{E}\big[\|y^{(m)}-x_L\|_2^2 \mid X\big]$
because unbiased messaging and uniform ordered-pair sampling give
$\mathbb{E}[y^{(m)}\mid X]=\bar{x}=\mathbb{E}[x_L\mid X]$.
Decompose $y^{(m)}-x_L=(y^{(m)}-x_S)+(x_S-x_L)$ to obtain the Soft term plus an extra variance term
$\frac{\alpha^2}{N^2}\mathbb{E}\|y^{(m)}-x_S\|_2^2$.
More explicitly, conditioning on $(x_S,x_L)$ gives
\[
\mathbb{E}\big[\langle y^{(m)}-x_S,x_S-x_L\rangle \mid x_S,x_L\big]=0,
\]
so
\[
\mathbb{E}\big[\|y^{(m)}-x_L\|_2^2 \mid X\big]
=
\mathbb{E}\big[\|x_S-x_L\|_2^2 \mid X\big]
+
\mathbb{E}\big[\|y^{(m)}-x_S\|_2^2 \mid X\big].
\]
By \eqref{eq:empirical_var_norm_app},
$\mathbb{E}\|y^{(m)}-x_S\|_2^2=\frac{1}{m}(1-\|x_S\|_2^2)$, yielding \eqref{eq:topm_drift}.

\subsection{Voter-model reduction and winner symmetry}
\label{app:winner_symmetry_proof}

With $\alpha=1$, the listener is overwritten by the sampled one-hot message: $x_L' = e_{k^\star}$ with $k^\star\sim\mathrm{Cat}(x_S)$. Under uniform random pair selection on the complete interaction graph, each agent is chosen as a listener infinitely often; hence after an almost-surely finite "coupon collector" time, all $x_i$ lie on simplex vertices and remain there. From that time onward the dynamics is exactly the $K$-state voter/Moran copying process on a finite complete graph, which almost surely reaches a consensus (an absorbing vertex state) in finite time.

For the winner distribution, fix a token $k$ and define the population mean coordinate $\bar{x}_k(t)\coloneqq \frac{1}{N}\sum_{i=1}^N x_{i,k}(t)$. In one update,
$\bar{x}_k(t+1)=\bar{x}_k(t)+\frac{1}{N}\big(y_k-x_{L,k}(t)\big)$ where $y=e_{k^\star}$.
Conditioned on the current state $\mathcal{F}_t$, $\mathbb{E}[y_k\mid\mathcal{F}_t]=\mathbb{E}[x_{S,k}(t)\mid\mathcal{F}_t]=\bar{x}_k(t)$ and also $\mathbb{E}[x_{L,k}(t)\mid\mathcal{F}_t]=\bar{x}_k(t)$, so $\mathbb{E}[\bar{x}_k(t+1)\mid\mathcal{F}_t]=\bar{x}_k(t)$; the expected value of $\bar{x}_k$ is unchanged at each step (equivalently, $(\bar{x}_k(t))$ is a bounded martingale). Let $T_{\text{cons}}$ be the (a.s.\ finite) consensus time; then $\bar{x}_k(T_{\text{cons}})=\mathbf{1}\{\text{winner}=k\}$. By optional stopping for bounded martingales, $\Pr[\text{winner}=k]=\mathbb{E}[\bar{x}_k(T_{\text{cons}})]=\bar{x}_k(0)$. If the initialization is exchangeable across labels (e.g.\ $x_i(0)=\mathbf{1}/K$), then $\bar{x}_k(0)=1/K$ and the winner is uniform over $\{1,\dots,K\}$.

\subsection{Diffusion approximation for weak asymmetry}
\label{app:weak_asymmetry_diffusion}
For $K=2$, let $\bar p=\frac{1}{N}\sum_i p_i$ with $p_i=x_{i1}$.
Under a small bias $h$, expand $\tilde p_S = p_S + h\,p_S(1-p_S)+\mathcal{O}(h^2)$.
Combining this with the QSG update and a mean-field approximation yields a one-dimensional diffusion for $\bar p$ in population rounds $\tau=t/N$,
\begin{equation}
d\bar p \;=\; \mu(\bar p)\,d\tau + \sqrt{D(\bar p)}\,dW,
\end{equation}
with drift $\mu(\bar p)\propto \alpha\,h\,\bar p(1-\bar p)$ and diffusion $D(\bar p)\propto (\alpha^2/(mN))\,\bar p(1-\bar p)$. Equivalently, per interaction step the mean increment scales as $\mathbb{E}[\Delta\bar p]\propto (\alpha/N)\,h\,\bar p(1-\bar p)$.
The fixation probability $\pi(p_0)$ of this diffusion surrogate solves the backward equation
$\mu(\bar p)\pi'(\bar p)+\tfrac12 D(\bar p)\pi''(\bar p)=0$ with $\pi(0)=0$, $\pi(1)=1$,
giving $\pi(p_0)=\frac{1-\exp(-2\Gamma_h p_0)}{1-\exp(-2\Gamma_h)}$ where $\Gamma_h=\frac{mN}{\alpha}h$.
From $p_0=1/2$, this yields the approximate logistic/tanh form used in Eq.~\eqref{eq:logistic_tanh_biased}.

\subsection{Order parameters and second-moment identities}
\label{app:order_parameters}

Recall the population mean
$
\bar{x} \coloneqq \frac{1}{N}\sum_{i=1}^N x_i
$
and polarization potential
$
U \coloneqq \|\bar{x}\|_2^2.
$
Define the disagreement energy
\begin{equation}
V \;\coloneqq\; \sum_{i=1}^N \|x_i-\bar{x}\|_2^2,
\label{eq:V_def_app}
\end{equation}
and the mean self-overlap
\begin{equation}
q \;\coloneqq\; \frac{1}{N}\sum_{i=1}^N \|x_i\|_2^2.
\label{eq:q_def_app}
\end{equation}
Finally, we define the coordination rate
\begin{equation}
S \;\coloneqq\; \frac{1}{N(N-1)}\sum_{i\neq j} x_i^\top x_j.
\label{eq:S_def_app}
\end{equation}

For any population state $X=(x_1,\dots,x_N)$, the basic second-moment identities are
\begin{align}
V &= N\,(q-U),
\label{eq:V_q_U_app}
\\
S &= U - \frac{V}{N(N-1)}.
\label{eq:S_U_V_app}
\end{align}
Expand $V=\sum_i \|x_i-\bar{x}\|_2^2=\sum_i\|x_i\|_2^2 - N\|\bar{x}\|_2^2$, giving $V=N(q-U)$.
Also $\|\sum_i x_i\|_2^2 = \sum_i\|x_i\|_2^2+\sum_{i\neq j}x_i^\top x_j$.
Since $\|\sum_i x_i\|_2^2=N^2U$ and $\sum_i\|x_i\|_2^2 = V+NU$, we obtain
$\sum_{i\neq j}x_i^\top x_j = N(N-1)U - V$, hence \eqref{eq:S_U_V_app}.

For a uniformly random ordered pair $(S,L)$ with $S\neq L$,
\begin{equation}
\mathbb{E}\big[\|x_S-x_L\|_2^2 \mid X\big] = \frac{2}{N-1}\,V.
\label{eq:pairdist_app}
\end{equation}
Because the diagonal terms vanish,
\[
\sum_{i\neq j}\|x_i-x_j\|_2^2
=
\sum_{i,j}\|x_i-x_j\|_2^2.
\]
Expanding the right-hand side around the mean gives
$\sum_{i,j}\|x_i-x_j\|_2^2 = 2N\sum_i\|x_i-\bar{x}\|_2^2 = 2NV$.
Averaging over the $N(N-1)$ ordered pairs gives \eqref{eq:pairdist_app}.

\subsection{Mean update and preservation in expectation}
\label{app:mean_martingale}

For any regime in which $\mathbb{E}[y\mid x_S]=x_S$ (Soft/Hard/Top-$m$),
\begin{equation}
\bar{x}'=\bar{x}+\frac{\alpha}{N}(y-x_L).
\label{eq:mean_update_app}
\end{equation}
Conditioned on the current state $X$,
\begin{equation}
\mathbb{E}[\bar{x}'\mid X]=\bar{x},
\end{equation}
so each coordinate $(\bar{x}_k(t))_{t\ge 0}$ is preserved in expectation from one step to the next (equivalently, it is a bounded martingale).
Only the listener changes: $x_L'-x_L=\alpha(y-x_L)$, so $\bar{x}'=\bar{x}+\frac{1}{N}(x_L'-x_L)$, yielding \eqref{eq:mean_update_app}.
Now condition on $X$. Since $(S,L)$ is a uniformly random ordered pair with $S\neq L$,
$\mathbb{E}[x_S\mid X]=\mathbb{E}[x_L\mid X]=\bar{x}$.
If $\mathbb{E}[y\mid x_S]=x_S$, then $\mathbb{E}[y\mid X]=\bar{x}$.
Thus $\mathbb{E}[y-x_L\mid X]=0$, implying $\mathbb{E}[\bar{x}'\mid X]=\bar{x}$.

\subsection[Message variance for Top-m]{Message variance for Top-$m$}
\label{app:topm_variance}

Recall the empirical message $y^{(m)}=\frac{1}{m}\sum_{j=1}^m e_{k_j}$ with $k_j\overset{\mathrm{iid}}{\sim}\mathrm{Cat}(x_S)$.

Conditioned on $x_S$,
\begin{align}
\mathbb{E}[y^{(m)}\mid x_S] &= x_S,
\\
\mathrm{Cov}(y^{(m)}\mid x_S) &= \frac{1}{m}\left(\mathrm{diag}(x_S)-x_Sx_S^\top\right),
\\
\mathbb{E}\big[\|y^{(m)}-x_S\|_2^2 \mid x_S\big] &= \frac{1}{m}\left(1-\|x_S\|_2^2\right).
\label{eq:empirical_var_norm_app}
\end{align}
Write $y^{(m)}=\frac{1}{m}\sum_{j=1}^m Y_j$ with $Y_j=e_{k_j}$ i.i.d.
Then $\mathbb{E}[Y_j\mid x_S]=x_S$ and $\mathrm{Cov}(y^{(m)}\mid x_S)=\frac{1}{m}\mathrm{Cov}(Y_1\mid x_S)$.
Since $\mathbb{E}[Y_1Y_1^\top\mid x_S]=\mathrm{diag}(x_S)$,
$\mathrm{Cov}(Y_1\mid x_S)=\mathrm{diag}(x_S)-x_Sx_S^\top$.
Taking trace gives \eqref{eq:empirical_var_norm_app}.

\subsection{Polarization drift: full derivation of the variance-injection law}
\label{app:U_drift_proof}

Let $U=\|\bar{x}\|_2^2$. Assume unbiased messaging $\mathbb{E}[y\mid x_S]=x_S$ (true for Soft/Hard/Top-$m$).
For the listener update $x_L'=(1-\alpha)x_L+\alpha y$,
\begin{equation}
\mathbb{E}[\Delta U \mid X] = \frac{\alpha^2}{N^2}\,\mathbb{E}\big[\|y-x_L\|_2^2 \mid X\big].
\label{eq:U_step_identity_app}
\end{equation}
From \eqref{eq:mean_update_app}, $\bar{x}'=\bar{x}+\frac{\alpha}{N}(y-x_L)$.
Expand $\|\bar{x}'\|_2^2-\|\bar{x}\|_2^2 = \frac{2\alpha}{N}\langle \bar{x},y-x_L\rangle + \frac{\alpha^2}{N^2}\|y-x_L\|_2^2$.
Conditioned on $X$, unbiased messaging and uniform ordered-pair sampling give
$\mathbb{E}[y\mid X]=\bar{x}=\mathbb{E}[x_L\mid X]$, so $\mathbb{E}[y-x_L\mid X]=0$ and the linear term vanishes.

\smallskip
\noindent\textit{Soft exchange.}
Under Soft exchange ($y=x_S$),
\begin{equation}
\mathbb{E}[\Delta U \mid X]_{\mathrm{soft}}
=
\frac{\alpha^2}{N^2}\mathbb{E}\big[\|x_S-x_L\|_2^2 \mid X\big]
=
\frac{2\alpha^2}{N^2(N-1)}\,V.
\label{eq:U_soft_V_app}
\end{equation}
This is \eqref{eq:U_step_identity_app} with \eqref{eq:pairdist_app} substituted for
$\mathbb{E}\big[\|x_S-x_L\|_2^2 \mid X\big]$.

Under Top-$m$ empirical communication $y=y^{(m)}$,
\begin{equation}
\mathbb{E}[\Delta U \mid X]_{\mathrm{top}m}
=
\mathbb{E}[\Delta U \mid X]_{\mathrm{soft}}
+
\frac{\alpha^2}{mN^2}\,\mathbb{E}\!\left[1-\|x_S\|_2^2 \mid X\right].
\label{eq:U_master_app}
\end{equation}
Start from \eqref{eq:U_step_identity_app}. Decompose
$y^{(m)}-x_L=(x_S-x_L)+(y^{(m)}-x_S)$.
Conditioned on $(x_S,x_L)$, $\mathbb{E}[y^{(m)}-x_S\mid x_S]=0$, so the cross term vanishes:
\[
\mathbb{E}\|y^{(m)}-x_L\|_2^2
=
\mathbb{E}\|x_S-x_L\|_2^2 + \mathbb{E}\|y^{(m)}-x_S\|_2^2.
\]
The first term is the Soft contribution \eqref{eq:U_soft_V_app}, and \eqref{eq:empirical_var_norm_app}
gives the second term.

\smallskip
\noindent\textit{Closed form in $(U,V)$.}
Using \eqref{eq:V_q_U_app}, so that $q = U+V/N$,
\begin{equation}
\mathbb{E}[\Delta U \mid X]_{\mathrm{top}m}
=
\frac{2\alpha^2}{N^2(N-1)}\,V
+
\frac{\alpha^2}{mN^2}\left(1-U-\frac{V}{N}\right).
\label{eq:U_topm_UV_app}
\end{equation}
At perfect symmetry ($x_i=\mathbf{1}/K$), $V=0$ and $U=1/K$, giving
$\mathbb{E}[\Delta U \mid X]_{\mathrm{top}m}=\frac{\alpha^2}{mN^2}(1-1/K)$.

\subsection{Disagreement drift and coordination drift}
\label{app:V_S_drift}

Let $V=\sum_i\|x_i-\bar{x}\|_2^2$. Under Top-$m$ empirical communication,
\begin{equation}
\mathbb{E}[\Delta V \mid X]
=
-\frac{2\alpha}{N-1}\left(1-\alpha+\frac{\alpha}{N}\right)\,V
+\frac{\alpha^2(N-1)}{mN}\left(1-U-\frac{V}{N}\right).
\label{eq:V_drift_app}
\end{equation}
Let $\Delta x \coloneqq x_L'-x_L=\alpha(y-x_L)$ and $\Delta\bar{x}=\Delta x/N$.
Write $\delta_i=x_i-\bar{x}$ so $\sum_i\delta_i=0$.
A direct expansion gives
$V' = V + 2\langle \delta_L,\Delta x\rangle + \frac{N-1}{N}\|\Delta x\|_2^2$.
For the linear term, condition on $(x_S,x_L)$ and use $\mathbb{E}[y\mid x_S]=x_S$ to replace
$\Delta x$ by $\alpha(x_S-x_L)$ in expectation. Averaging over uniformly sampled ordered pairs then gives the contraction contribution
$-\frac{2\alpha}{N-1}\left(1-\alpha+\frac{\alpha}{N}\right)V$.
For the quadratic term, decompose $\|y-x_L\|_2^2=\|x_S-x_L\|_2^2+\|y-x_S\|_2^2$ and apply
\eqref{eq:pairdist_app} and \eqref{eq:empirical_var_norm_app} to obtain the injection term.
Finally use $1-\mathbb{E}\|x_S\|_2^2 = 1-q = 1-U-V/N$.

\smallskip
\noindent\textit{Soft limit.}
Under Soft exchange ($m=\infty$), the injection term vanishes and
\begin{equation}
\mathbb{E}[\Delta V \mid X]_{\mathrm{soft}}
=
-\frac{2\alpha}{N-1}\left(1-\alpha+\frac{\alpha}{N}\right)\,V
\le 0.
\end{equation}

\smallskip
\noindent\textit{Coordination drift.}
Using \eqref{eq:S_U_V_app} and the drift formulas above,
\begin{equation}
\mathbb{E}[\Delta S \mid X]
=
\frac{2\alpha}{N(N-1)^2}\,V
\;\ge\; 0.
\label{eq:S_drift_app}
\end{equation}
In particular, $\mathbb{E}[\Delta S \mid X]$ is nonnegative and depends on $m$ only through the evolution of $V$.

\subsection{Homogeneous mean-field closure and consensus time scaling}
\label{app:meanfield}

A common closure assumes agents remain approximately homogeneous: $x_i\approx \bar{x}$, so $V\approx 0$ and $q\approx U$.
Then \eqref{eq:U_topm_UV_app} yields the per-step approximation
\begin{equation}
\mathbb{E}[\Delta U \mid X] \approx \frac{\alpha^2}{mN^2}(1-U).
\label{eq:mf_step_app}
\end{equation}
Measuring time in population rounds $\tau=t/N$ and treating the dynamics continuously gives
\begin{equation}
\frac{dU}{d\tau} \approx \frac{\alpha^2}{mN}(1-U),
\qquad
U(\tau)\approx 1-(1-U_0)\exp\!\left(-\frac{\alpha^2}{mN}\tau\right).
\label{eq:mf_solution_app}
\end{equation}
Thus the characteristic timescale in population rounds scales as
$\tau_{\mathrm{cons}}\sim \frac{mN}{\alpha^2}$.

\subsection[Mean-field comparison at large alpha]{Mean-field comparison at large $\alpha$}
\label{app:meanfield_alpha}

The mean-field approximation~\eqref{eq:mf_solution_app} is derived as a continuous approximation to the discrete QSG dynamics.
This approximation is most accurate when the adaptation rate $\alpha$ is small, so that each update produces an infinitesimal change.
At $\alpha=1$ (the naming-game limit where the listener is completely overwritten), the discrete dynamics dominate.

\paragraph{Why does simulation exceed theory at large $\alpha$?}
At small $\alpha$, the continuous ODE approximation is accurate, and a secondary effect becomes visible: agent heterogeneity.
By Jensen's inequality, $\mathbb{E}[\|x_S\|_2^2] \geq \|\bar{x}\|_2^2 = U$, so the actual drift term $(1-\mathbb{E}[\|x_S\|_2^2])$ is smaller than the mean-field approximation $(1-U)$.
This causes simulation to sit \emph{below} the theory curve (Fig.~\ref{fig:meanfield_validation}).

At large $\alpha$, however, the discrete dynamics dominate.
With $\alpha=1$, listeners snap to one-hot vectors after a single interaction.
These large jumps produce faster polarization than the smooth exponential predicted by the ODE, causing simulation to sit \emph{above} the theory curve.

\paragraph{Implication.}
At small $\alpha$, agent heterogeneity makes the effective drift smaller than the homogeneous mean-field approximation, so the mean-field curve typically \emph{overestimates} polarization speed (and thus \emph{underestimates} time to reach a fixed consensus threshold).
At large $\alpha$, discrete one-hot overwrites accelerate polarization relative to the smooth ODE, so the mean-field curve typically \emph{underestimates} polarization speed (and thus \emph{overestimates} time to threshold).

\subsection[Tempered sampling and the crossover parameter Gamma T]{Tempered sampling and the crossover parameter $\Gamma_T$}
\label{app:Gamma}

For the tempered-sampling extension used in Fig.~\ref{fig:luckiest_vs_fittest}, we define the temperature transform
\begin{equation}
g_T(x)_k \propto x_k^{1/T},
\qquad \sum_{k=1}^K g_T(x)_k = 1,
\label{eq:tempering_app}
\end{equation}
and generate messages from $g_T(x_S)$ instead of $x_S$ (Hard/Top-$m$).
Then $\mathbb{E}[y\mid x_S]=g_T(x_S)$, and the mean-field dynamics of $\bar{x}$ becomes
\begin{equation}
\frac{d\bar{x}}{d\tau} \approx \alpha\big(g_T(\bar{x})-\bar{x}\big).
\label{eq:mf_mean_tempered_app}
\end{equation}

Linearizing around the symmetric point $u=\mathbf{1}/K$ with $\bar{x}=u+\delta$ and $\sum_k\delta_k=0$ gives
\begin{equation}
g_T(u+\delta)=u+\frac{1}{T}\delta+O(\|\delta\|_2^2),
\qquad
\Rightarrow\qquad
\frac{d\delta}{d\tau}\approx \alpha\left(\frac{1}{T}-1\right)\delta.
\label{eq:linearization_T_app}
\end{equation}
Thus $T<1$ deterministically amplifies small asymmetries while $T>1$ damps them.

Quantized communication injects sampling-driven polarization at a characteristic population-round rate scale $\sim \alpha^2/(mN)$
(cf.\ \eqref{eq:mf_solution_app} near symmetry).
Comparing the deterministic linear rate $\alpha|1/T-1|$ to the quantization-driven scale $\alpha^2/(mN)$ motivates
the dimensionless crossover parameter
\begin{equation}
\Gamma_T \;\coloneqq\; \frac{mN}{\alpha}\left|\frac{1}{T}-1\right|.
\label{eq:Gamma_app}
\end{equation}
We interpret $\Gamma_T\approx 1$ as a \emph{finite-size crossover} between near-neutral (drift-dominated) and tempering-dominated regimes,
not as a literal thermodynamic phase transition of the finite-$N$ absorbing chain.

\subsection[Relating U to entropy and magnetization via a one-vs-rest ansatz]{Relating $U$ to entropy and magnetization via a one-vs-rest ansatz}
\label{app:U_to_HM}

To connect second-moment theory to entropy/magnetization plots, a useful approximation is the one-vs-rest ansatz
\[
\bar{x}\approx \left(p,\frac{1-p}{K-1},\dots,\frac{1-p}{K-1}\right),
\qquad p=\max_k \bar{x}_k \in[1/K,1].
\]
Under this ansatz,
\begin{equation}
U = p^2 + \frac{(1-p)^2}{K-1},
\qquad
p(U)=\frac{1+\sqrt{(K-1)(KU-1)}}{K}.
\label{eq:p_of_U_app}
\end{equation}
The magnetization $M=\frac{Kp-1}{K-1}$ becomes
\begin{equation}
M(U)=\sqrt{\frac{KU-1}{K-1}}.
\label{eq:M_of_U_app}
\end{equation}
The normalized entropy $H(\bar{x})=-\frac{1}{\log K}\sum_k \bar{x}_k\log(\bar{x}_k+\varepsilon)$ yields
\begin{equation}
H(U)= -\frac{1}{\log K}\left[
p(U)\log p(U) + (1-p(U))\log\!\left(\frac{1-p(U)}{K-1}\right)
\right],
\label{eq:H_of_U_app}
\end{equation}
where $\varepsilon$ is used only for numerical stability in implementations.

\section{Experiments}
\label{app:experiments_overview}

\subsection{Numerical simulations of QSG}
\label{app:toy_sims}
We report QSG simulations comparing Soft, Hard, and Top-$m$ dynamics across symmetry breaking, temperature sweeps, and entropy trajectories. The one-step drift identity test in Fig.~\ref{fig:thm2_validation} is reported in the main text; here we focus on additional dynamics and robustness sweeps.

\subsubsection{Setup}
Unless stated otherwise:
\begin{itemize}[leftmargin=1.2em]
\item Population size $N=24$, vocabulary size $K=10$ (matching a common experimental setting in \cite{Ashery2025SciAdv}).
\item Soft update: $x_L\leftarrow (1-\alpha)x_L+\alpha x_S$.
\item Hard update: $k^\star\sim\mathrm{Cat}(x_S)$ and $x_L\leftarrow (1-\alpha)x_L+\alpha e_{k^\star}$.
\item Top-$m$ update: sample $m$ tokens i.i.d.\ with replacement and update toward their empirical distribution.
\item We track coordination rate $S(t)$ (Eq.~\eqref{eq:S_def_app}) and entropy $H(t)$.
\end{itemize}

\subsubsection{Simulation suite}
We run four experiments: (i) Soft versus Hard dynamics from the symmetric initialization $x_i(0)=\mathbf{1}/K$,
(ii) a softmax-temperature sweep for Hard sampling, showing a finite-horizon crossover near $T\approx 1$,
(iii) a comparison of Hard, Soft, and Top-$m$ updates from symmetric initialization, where entropy remains high for $m>1$ while Hard collapses,
and (iv) the direct Monte Carlo test of Theorem~\ref{thm:topm_scaling} reported in Fig.~\ref{fig:thm2_validation}.

\subsubsection{Method summary and reproducibility}
\label{app:experiments}
We simulate QSG directly in probability space with Soft, Hard, and Top-$m$ updates as defined in the main text.
The Top-$m$ message uses the empirical distribution from $m$ independent samples with replacement.
The Monte Carlo estimator samples a random speaker and listener, computes the one-step change in $U=\|\bar{x}\|_2^2$, and averages over repeated draws.
Unless noted, we use $N=24$, $K=10$, and the $\alpha$ listed in each figure; $m$ ranges over $\{1,2,3,5,10,20\}$.
Code uses NumPy with fixed seeds.

\subsection{LLM population experiments}
\label{app:llm_population}
LLM population experiments using the Neutral Naming Drift (NND) protocol are reported in Sec.~\ref{sec:experiments} and summarized in Fig.~\ref{fig:theory_vs_llm}.
Across all LLM runs, we use delayed-reveal interactions with memory size $H=10$, synthetic 5-character labels, a fixed referent string, and no external reward.
Label order is shuffled each interaction, probes are measurement-only, and the consensus threshold is $0.9$.
Prompt structure and decoding are fixed across runs: the system message enforces valid JSON output over the allowed labels, and the speaker/listener prompts specify the referent, allowed labels, and the agent's memory of past interactions.
Fig.~\ref{fig:prompt_example} shows a concrete example.

\begin{figure}[t]
\centering
\small
\begin{tcolorbox}[colback=black!3, colframe=black!30, boxrule=0.4pt, arc=2pt,
  left=6pt, right=6pt, top=5pt, bottom=5pt, title={\footnotesize\textbf{System message}}, fonttitle=\sffamily]
\raggedright\ttfamily\scriptsize
You must output only valid JSON. No extra keys, no explanations, no markdown.\\
Valid labels are exactly those in Allowed labels. Never output "<PAD>".
\end{tcolorbox}
\vspace{-4pt}
\begin{tcolorbox}[colback=black!3, colframe=black!30, boxrule=0.4pt, arc=2pt,
  left=6pt, right=6pt, top=5pt, bottom=5pt, title={\footnotesize\textbf{User message (speaker)}}, fonttitle=\sffamily]
\raggedright\ttfamily\scriptsize
Referent: ref\_07\\
Allowed labels: ["hdsad", "vokhg", "fmhja"]\\
The list order is randomized and has no meaning.\\
Both players are choosing a label for the same referent in repeated interactions. The memory shows labels you observed from previous interactions with partners.\\
Memory (last H observed messages, oldest -> newest, padded with "<PAD>"): ["<PAD>", "<PAD>", "<PAD>", "<PAD>", "<PAD>", "<PAD>", "<PAD>", "fmhja", "hdsad", "vokhg"]\\
Each memory entry is a label string.\\[4pt]
Constraints:\\
- Output JSON only.\\
- Every label must be from Allowed labels.\\[4pt]
Output JSON exactly: \mbox{\texttt{\detokenize{{"label": "<label>"}}}}
\end{tcolorbox}
\caption{\textbf{Example speaker prompt (Hard, $m=1$, $K=3$, $H=10$).}
Labels are synthetic 5-character strings; memory is zero-padded with \texttt{<PAD>} tokens.
Under the delayed-reveal protocol used in all experiments, the listener prompt follows the same basic structure. The listener does not observe the speaker's current output but updates its memory with the speaker's label after responding.}
\label{fig:prompt_example}
\end{figure}

For the $K=3$ N-sweeps in Fig.~\ref{fig:theory_vs_llm} (GPT-4o and Claude Haiku 4.5), we use
$N\in\{2,3,4,6,8,12,16,24,32\}$,
$m=1$ (Hard),
and horizon $T=30N$ interactions.
Probes are taken once per population round, with one probe sample per agent.
Each $(\text{model},N)$ point aggregates 5 trials.

For the $K=10$ Top-$m$ sweep in Fig.~\ref{fig:m_sweep_llm_scaling}, we use GPT-4o with
$N=4$, horizon $t_{\max}=120$ interaction steps, and $m\in\{1,2,3,4\}$ (10 trials each).
\label{app:topm_llm_details}
Early drift is estimated from the initial linear regime using measurement-only probes at each interaction step; error bars in Fig.~\ref{fig:m_sweep_llm_scaling} show s.e.m.

\end{document}